%% file: iclr2022.tex
\documentclass{article}
\usepackage{iclr2022_conference, times}
\iclrfinaltrue

\usepackage{hyperref}       
\usepackage{url}            
\usepackage{booktabs}       
\usepackage{amsfonts}       
\usepackage{nicefrac}       
\usepackage{microtype}      
\usepackage{natbib} 
\setcitestyle{square,numbers}

\usepackage{mathtools} 
\usepackage{booktabs} 
\usepackage{tikz} 

\usepackage{amsthm}
\usepackage{amssymb} 
\usepackage{latexsym}
\usepackage{makecell}
\usepackage{multirow}
\usepackage{stfloats}
\usepackage{subcaption}

\usepackage{algorithm}
\usepackage{algpseudocode}

\hypersetup{
    colorlinks=true,
    linkcolor=blue,
    filecolor=magenta,      
    urlcolor=blue,
    citecolor=blue,
}
\newcommand{\mc}{monotonicity constraint }

\title{Rethinking the Implementation Tricks and Monotonicity Constraint in Cooperative Multi-Agent Reinforcement Learning}
\author{%
  Jian Hu \thanks{Jian Hu and Siyang Jiang contributed equally to this work.} \thanks{Corresponding author.} \\
  Graduate Institute of Networking and Multimedia\\
  National Taiwan University\\
  Taipei\\
  \texttt{janhu9527@gmail.com} \\
   \And
  Siyang Jiang $^*$  \\
  Graduate Institute of Electrical Engineering\\
  National Taiwan University\\
  Taipei\\
  \texttt{syjiang@arbor.ee.ntu.edu.tw} \\
   \And
  Seth Austin Harding  \\
  Department of Computer Science\\
  National Taiwan University\\
  Taipei\\
  \texttt{b06902101@ntu.edu.tw} \\
    \And
  Haibin Wu  \\
  Graduate Institute of Communication Engineering\\
  National Taiwan University\\
  Taipei\\
  \texttt{f07921092@ntu.edu.tw} \\
     \And
  Shih-wei Liao  \\
  Department of Computer Science\\
  National Taiwan University\\
  Taipei\\
  \texttt{liao@csie.ntu.edu.tw} \\
}

\begin{document}
\maketitle

\input{sections/abstract}
\input{sections/introduction}
\input{sections/preliminaries}

\input{sections/relatedworks}

\input{sections/ExperimentsSetup/main}

\input{sections/Experiment/main}

\input{sections/discussion}
\input{sections/conclusion}

\clearpage
\bibliography{iclr2022_conference}
\bibliographystyle{iclr2022_conference}

\clearpage
\appendix
\input{appendix/main}


\end{document}

%% file: sections/abstract.tex
\begin{abstract}
Many complex multi-agent systems such as robot swarms control and autonomous vehicle coordination can be modeled as Multi-Agent Reinforcement Learning (MARL) tasks. QMIX, a widely popular MARL algorithm, has been used as a baseline for the benchmark environments, e.g., Starcraft Multi-Agent Challenge (SMAC), Difficulty-Enhanced Predator-Prey (DEPP). Recent variants of QMIX target relaxing the monotonicity constraint of QMIX, allowing for performance improvement in SMAC. In this paper, we investigate the code-level optimizations of these variants and the monotonicity constraint. (1) We find that such improvements of the variants are significantly affected by various code-level optimizations. (2) The experiment results show that QMIX with normalized optimizations outperforms other works in SMAC; (3) beyond the common wisdom from these works, the monotonicity constraint can improve sample efficiency in SMAC and DEPP. We also discuss why monotonicity constraints work well in purely cooperative tasks with a theoretical analysis. We open-source the code at \url{https://github.com/hijkzzz/pymarl2}.
\end{abstract}


%% file: sections/introduction.tex
\section{Introduction}
Multi-agent cooperative games have many complex real-world applications such as, robot swarm control~\cite{huttenrauch2017guided,xiao2020macro,nachum2019multi}, autonomous vehicle coordination~\cite{cao2012overview, zhou2020smarts}, and sensor networks~\cite{zhang2011coordinated}, a complex task always requires multi-agents to accomplish together. Multi-Agent Reinforcement Learning (MARL), is used to solve the multi-agent systems tasks~\cite{xiao2020macro}. 

In multi-agent systems, a typical challenge is a limited scalability and inherent constraints on agent observability and communication. Therefore, decentralized policies that act only on their local observations are necessitated and widely used~\citep{zhou2020learning}. Learning decentralized policies is an intuitive approach for training agents independently. However, simultaneous exploration by multiple agents often results in non-stationary environments, which leads to unstable learning. Therefore, \textit{Centralized Training and Decentralized Execution} (CTDE)~\citep{kraemer2016multiagent} allows for independent agents to access additional state information that is unavailable during policy inference.

Many CTDE learning algorithms have been proposed for the better sample efficiency in cooperative tasks\citep{wei2018multiagent}. Among them, several value-based approaches achieve state-of-the-art (SOTA) performance~\citep{rashid2018qmix,wang2020qplex,yang2020qatten,rashid2020weighted} on such benchmark environments, e.g., Starcraft Multi-Agent Challenge (SMAC)~\citep{samvelyan2019starcraft}, Predator-Prey (PP)~\citep{boehmer2020dcg,peng2020facmac}. To enable effective CTDE for multi-agent Q-learning, the Individual-Global-Max (IGM) principle~\citep{son2019qtran}  of equivalence of joint greedy action and individual greedy actions is critical. The primary advantage of the IGM principle is that it ensures consistency of policy with centralized training and decentralized execution. To ensure IGM principle, QMIX~\citep{rashid2018qmix} was proposed for factorizing the joint action-value function with the \textit{Monotonicity Constraint}~\cite{wang2020qplex}, however, limiting the expressive power of the mixing network.

To improve the performance of QMIX, some variants of QMIX \footnote{These algorithms are based on the mixing network from QMIX, so we call the variants of QMIX.}, including value-based approaches~\citep{yang2020qatten, rashid2020weighted, wang2020qplex, son2020qtran} and a policy-based approach~\citep{zhou2020learning}, have been proposed with the aim to relax the monotonicity constraint of QMIX.  However, while investigating the codes of these variants, we find that their performance is significantly affected by their code-level optimizations (or implementation tricks). Therefore, it is left unclear whether \mc indeed impairs the QMIX's performance.

In this paper, we investigate the impact of the \textit{code-level optimizations} and the \textit{monotonicity constraint} in cooperative MARL. Firstly, we investigate the effects of code-level optimizations, which enable QMIX to solve the most difficult challenges in SMAC. Afterward, we normalize the optimizations of QMIX and its variants; specifically, we perform the same hyperparameter search pattern for all algorithms, which includes using or removing a certain optimization and a grid hyperparameter search; the experiment results (Sec.~\ref{section:QMIX with tricks}) demonstrate that QMIX outperforms the other variants. Secondly, to study the impact of the monotonicity constraint, we propose a policy-based algorithm, RIIT; the experimental results (Sec.~\ref{phasic}) show that the monotonicity constraint improves sample efficiency in SMAC and DEPP. Lastly, to generalize cooperative tasks beyond SMAC and DEPP, we give a strict definition of purely cooperative tasks and a discussion about why monotonicity constraints work well in purely cooperative tasks.

To our best knowledge, this work is the first to analyze the monotonicity constraint and code-level optimizations in MARL. Our analysis shows that QMIX works well if a multi-agent task can be interpreted as purely cooperative, even if it can also be interpreted as competitive.

%% file: sections/preliminaries.tex
\section{Preliminaries}
\label{bg}

\textbf{Dec-POMDP.} We model the multi-robot RL problem as decentralized partially observable Markov decision process (Dec-POMDP) \citep{ong2009pomdps}, which composed of a tuple $G=\langle \mathcal{S}, \mathcal{U}, P, r, \mathcal{Z}, O, N, \gamma\rangle$. $s \in S$ describes the true state of the environment. At each time step, each agent $i \in \mathcal{N}:=\{1, \ldots, N\}$ chooses an action $u^{i} \in \mathcal{U}$, forming a joint action $\boldsymbol{u} \in  \mathcal{U}^{N}$. All state transition dynamics are defined by function $P\left(\boldsymbol{s}^{\prime} \mid \boldsymbol{s}, \boldsymbol{u}\right): \mathcal{S} \times \mathcal{U}^{N} \times \mathcal{S} \mapsto[0,1]$. Each agent has independent observation $z \in \mathcal{Z}$, determined by observation function $O(s, i): \mathcal{S} \times \mathcal{N} \mapsto \mathcal{Z}$. All agents share the same reward function $r(s, \boldsymbol{u}): \mathcal{S} \times \mathcal{U}^{N} \rightarrow \mathbb{R}$ and $\gamma \in[0,1)$ is the discount factor. The objective
function, shown in Eq.~\ref{jointobject}, is to maximize the joint value function to find a joint policy $\boldsymbol{\pi} = \langle \pi_{1},... ,\pi_{n}\rangle$.
\begin{eqnarray}\label{jointobject}
J\left(\pi\right)=\mathbb{E}_{u^{1} \sim \pi^{1}, \ldots, u^{N} \sim \pi^{N}, s \sim T}\left[\sum_{t=0}^{\infty} \gamma_{t} r_{t}\left(s_{t}, u_{t}^{1}, \ldots, u_{t}^{N}\right)\right]
\end{eqnarray}

\textbf{Centralized Training and Decentralized Execution (CTDE).}
CTDE is a popular paradigm~\citep{wang2020qplex} which  allows for the learning process to utilize additional state information~\citep{kraemer2016multiagent}. Agents are trained in a centralized way, i.e., learning algorithms, to access all local action observation histograms, global states, and sharing gradients and parameters. In the execution stage, each individual agent can only access its local action observation history $\tau^{i}$. 

\textbf{QMIX and Monotonicity Constraint.} To resolve the credit assignment problem in multi-agent learning, QMIX \citep{rashid2018qmix} learns a joint action-value function $Q_{tot}$ which can be represented in Eq.~\ref{sections:Q_total_1}:

\begin{eqnarray}
\begin{aligned}
\label{sections:Q_total_1}
Q_{tot}(s, \boldsymbol{u} ; \boldsymbol{\theta}, \phi)
= &g_{\phi}\left(s, Q_{1}\left(\tau^{1}, u^{1} ; \theta^{1}\right), \ldots, Q_{N}\left(\tau^{N}, u^{N} ;  \theta^{N}\right)\right) \\
\label{sections:Q_total_2}
&\frac{\partial Q_{tot}(s, \boldsymbol{u} ; \boldsymbol{\theta}, \phi)}{\partial Q_{i}\left(\tau^{i}, u^{i}; \theta^{i}\right)} \geq 0, \quad \forall i \in \mathcal{N}
\end{aligned}
\end{eqnarray}



where $\phi$ is the trainable parameter of the monotonic mixing network, which is a mixing network with \mc, and $\theta^i$ is the parameter of the agent network $i$. Benefiting from the monotonicity constraint in Eq. \ref{sections:Q_total_2}, maximizing joint $Q_{tot}$ is precisely the equivalent of maximizing individual $Q_i$, resulting in and allowing for optimal individual action to maintain consistency with optimal joint action. QMIX learns by sampling a multitude of transitions from the replay buffer and minimizing the mean squared temporal-difference (TD) error loss:

\begin{eqnarray}
\mathcal{L}(\theta)= \frac{1}{2} \sum_{i=1}^{b}\left[\left(y_{i}^{}-Q_{tot}(s, u ; \theta, \phi)\right)^{2}\right]
\end{eqnarray}
 
where the TD target value $y=r+\gamma \max _{u^{\prime}} Q_{tot}\left(s^{\prime}, u^{\prime} ; \theta^{-}, \phi^{-}\right)$ and $\theta^{-}$, $\phi^{-}$ are the target network parameters copied periodically from the current network and kept constant for a number of iterations. 
\input{table/ShortageOfQMIX}
However, the monotonicity constraint limits the mixing network's expressiveness, which may fail to learn in non-monotonic cases \citep{mahajan2020maven} \citep{rashid2020weighted}. Table \ref{nonmatrix} shows a non-monotonic matrix game that violates the monotonicity constraint. This game requires both robots to select the first action 0 (actions are indexed from top to bottom, left to right) in order to catch the reward 12; if only one robot selects action 0, the reward is -12. QMIX may learn an incorrect $Q_{tot}$ which has an incorrect argmax action as shown in Table~\ref{nonmatrix_qmix}. 



%% file: table/ShortageOfQMIX.tex
\begin{table}[hbtp]
 \begin{minipage}{0.48\columnwidth}
  \centering
      \begin{tabular}{|p{.5cm}|p{.5cm}|p{.5cm}|}
        \hline \ \textbf{12} & -12 & -12 \\
        \hline -12 & 0 & 0 \\
        \hline -12 & 0 & 0 \\
        \hline
    \end{tabular}
      \subcaption{Payoff matrix}
      \label{nonmatrix}
  \end{minipage}
  \begin{minipage}{0.48\columnwidth}
  \centering
         \begin{tabular}{|p{.5cm}|p{.5cm}|p{.5cm}|}
        \hline -12 & -12 & -12 \\
        \hline -12 & 0 & \textbf{0} \\
        \hline -12 & 0 & 0 \\
        \hline
        \end{tabular}
        \subcaption{QMIX: $Q_{tot}$ }
        \label{nonmatrix_qmix}
  \end{minipage}
  \caption{A non-monotonic matrix game. Bold text indicates the reward of the argmax action.}
\end{table}

%% file: sections/relatedworks.tex
\section{Related Works} \label{related}
In this section, we describe the variants of QMIX which relaxing the monotonicity constraint. We explain the details of these algorithms and show the code resources in Appendix~\ref{supplementary materials}.

\textbf{Value-based Methods} To enhance the expressive power of QMIX, Qatten \cite{yang2020qatten} introduces an attention mechanism to enhance the expression of QMIX; QPLEX \cite{wang2020qplex} transfers the monotonicity constraint from Q values to Advantage values \cite{mnih2016asynchronous}; QTRAN++ \cite{son2020qtran} and WQMIX \cite{rashid2020weighted} further relax the monotonicity constraint through a true value network and some theoretical constraints; however, Value-Decomposition Networks (VDNs)~\cite{sunehag2017valuedecomposition} only requires a linear decomposition where $Q_{tot} = \sum_{i}^{N} Q_i$, which can be seen as strengthening the monotonicity constraint. 

\textbf{Policy-based Methods} LICA \cite{zhou2020learning} completely removes the monotonicity constraint through a policy mixing critic. For other MARL policy-based methods, VMIX~\cite{su2020value} combines the Advantage Actor-Critic (A2C) \cite{stooke2018accelerated} with QMIX to extend the monotonicity constraint to value networks, i.e., replacing the value network (not Q value network) with the monotonic mixing network. DOP \cite{wang2020offpolicy} learns the policy networks using the Counterfactual Multi-Agent Policy Gradients (COMA) \cite{foerster2017counterfactual} with the $Q_i$ decomposed by QMIX. At last, we briefly describe the properties of these algorithms in Table \ref{tb:properties}. 

\input{table/all_algorithms}

All these algorithms show that their performance exceeds QMIX in SMAC, yet we find that they do not consider the impact of various code-level optimizations (Appendix \ref{appendix:basic_tricks}) in the implementations. \textbf{  Moreover, the performance of these algorithms is not even consistent in these papers.} For example, in papers \cite{wang2020qplex} and \cite{wang2020offpolicy}, QPLEX and DOP outperform QMIX, while in paper \cite{peng2020facmac}, both QPLEX and DOP underperform QMIX.

%% file: table/all_algorithms.tex
\begin{table}[htbp] \footnotesize
\centering
\begin{tabular}{ccccc}
\hline
Algorithms & Type         & Attention & Monotonic Constraint & Off-policy \\ \hline
VDNs       & Value-based  & No        & Very Strong          & Yes        \\
QMIX       & Value-based  & No        & Strong               & Yes        \\
Qatten     & Value-based  & Yes       & Strong               & Yes        \\
QPLEX      & Value-based  & Yes       & Medium               & Yes        \\
WQMIX      & Value-based  & No        & Weak                 & Yes        \\
VMIX       & Policy-based & No        & Strong               & No         \\
LICA       & Policy-based & No        & No                   & No         \\
RIIT        & Policy-based & No        & Strong               & Yes        \\ \hline
\end{tabular}
\caption{Properties of coopertive MARL algorithms.}
\label{tb:properties}
\end{table}

%% file: sections/ExperimentsSetup/main.tex
\section{Experiments Setup} \label{design}
To facilitate the study of cooperation in complex multi-robots scenarios,, we simulate the interaction among the robots through two hard computer games.

\input{sections/ExperimentsSetup/benchmark}
\input{sections/ExperimentsSetup/metric}

%% file: sections/ExperimentsSetup/benchmark.tex
\subsection{Benchmark Environment}

\textbf{StarCraft Multi-Agent Challenge (SMAC)} is used as our main benchmark testing environment, which is a ubiquitously-used multi-agent cooperative control environment for MARL algorithms~\citep{wang2020qplex,rashid2018qmix,son2020qtran,rashid2020weighted}. SMAC consists of a set of StarCraft II micro battle scenarios, whose goals are for allied robots to defeat enemy robots, and it classifies micro scenarios into \textit{Easy}, \textit{Hard}, and \textit{Super Hard} levels. The simplest VDNs~\citep{sunehag2017valuedecomposition} can effectively solve the Easy scenarios. It is worth noting that QMIX and VDNs achieves a 0\% win rate in three Super Hard scenarios $corridor$, $3s5z\_vs\_3s5z$, and $6h\_vs\_8z$~\citep{samvelyan2019starcraft}. Therefore, we mainly investigate the Hard and Super Hard scenarios in SMAC. 

\textbf{Difficulty-Enhanced Predator-Prey (DEPP)}
In vanilla Predator-Prey \citep{lowe2020multiagent}, three cooperating agents control three robot predators to chase a faster robot prey (the prey acts randomly). The goal is to capture the prey with the fewest steps possible. We leverage two difficulty-enhanced Predator-Prey variants to test the algorithms: (1) the first variant of Predator-Prey (PP) \cite{boehmer2020dcg} requires two predators to catch the prey at the same time to get a reward; (2) In the Continuous Predator-Prey (CPP) \cite{peng2020facmac}, the prey’s policy is replaced by a hard-coded heuristic policy, i.e., at any time step, moving the prey to the sampled position with the largest distance to the closest predator, 

%% file: sections/ExperimentsSetup/metric.tex
\subsection{Evaluation Metric}
Our primary evaluation metric is the function that maps the steps for the environment observed throughout the training to the median winning percentage (episode return for Predator-Prey) of the evaluation. Just as in QMIX~\citep{rashid2018qmix}, we repeat each experiment with several independent training runs (five independent random experiments). To accurately evaluate the convergence performance of each algorithm, \textbf{eight rollout processes for parallel sampling} are used to obtain as many samples as possible from the environments at a high rate. Specifically, our experiments can collect 10 million samples within 9 hours with a Core i7-7820X CPU and a GTX 1080 Ti GPU.

%% file: sections/Experiment/main.tex
\section{Experiments}
\label{section:experiments}
Our experiments consist of two parts. The first part demonstrates the performance of several isolated tricks from the variants. The second part is the reconceptualization of the monotonicity constraint. 

\input{sections/Experiment/ImplementationTricks/main}
\input{sections/Experiment/Rethinking/main}

%% file: sections/Experiment/ImplementationTricks/main.tex
\subsection{Rethinking the Code-level Optimizations} \label{experi} 
The code-level optimizations are the tricks unaccounted for in the experimental design, but that might hold significant effects on the result. To better understand their influences on performance, we perform ablation experiments on these tricks incrementally and provide some suggestions for tuning. We study the major optimizations here, and introduce the other tricks in Appendix~\ref{appendix:basic_tricks}.


\input{sections/Experiment/ImplementationTricks/optimization}
\input{sections/Experiment/ImplementationTricks/traces}
\input{sections/Experiment/ImplementationTricks/buffersize}
\input{sections/Experiment/ImplementationTricks/processnumber}
\input{sections/Experiment/ImplementationTricks/hidden_size}
\input{sections/Experiment/ImplementationTricks/step}
\input{sections/Experiment/ImplementationTricks/overall}

%% file: sections/Experiment/ImplementationTricks/optimization.tex
\subsubsection{Optimizer}  \label{optimization}

\textbf{Study description.} QMIX and the majority of its variant algorithms use RMSProp to optimize neural networks as they prove stable in SMAC. We attempt to use Adam to optimize QMIX's neural network with quickly convergence benefiting from momentum:

\input{fig/optimization1}

\textbf{Interpretation.} Figure \ref{fig:optimization1} shows that Adam \citep{kingma2014adam} increases the win rate by 100\% on the Super Hard map $corridor$. Adam boosts the network's convergence allowing for full utilization of the large quantity of samples sampled in parallel. We find that the Adam optimizer solves the problem posed by \citep{su2020value} in which QMIX does not work well under parallel training.

\textbf{Recommendation.} Use Adam with parallel training.

%% file: fig/optimization1.tex
\begin{figure}[h]
\centering
\includegraphics[height=2.6cm]{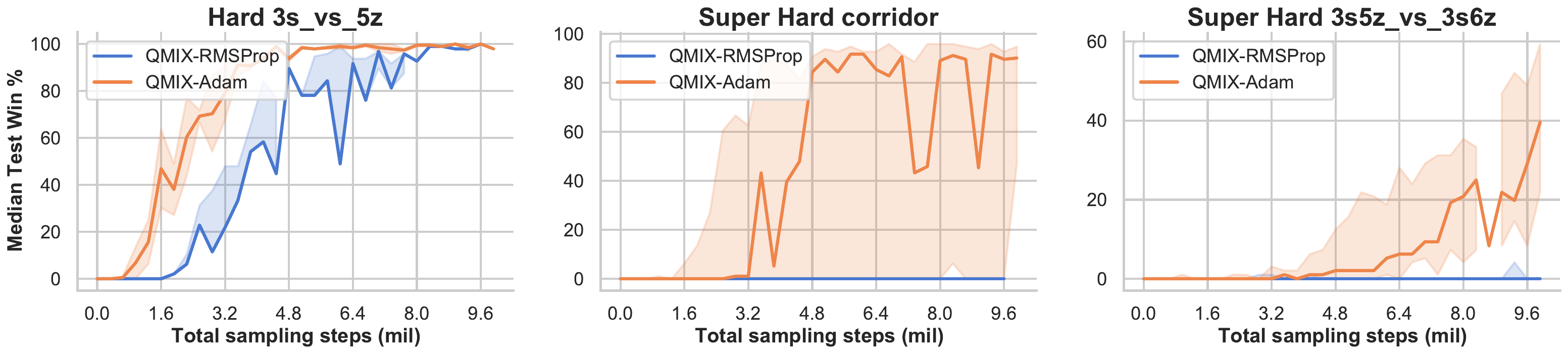}
\caption{Adam significantly improves performance when samples are updated quickly.}
\label{fig:optimization1}
\end{figure}

%% file: sections/Experiment/ImplementationTricks/traces.tex
\subsubsection{Eligibility Traces} 

\textbf{Study description.} 
Eligibility traces such as TD$(\lambda)$~\citep{sutton2018reinforcement}, Peng's Q$(\lambda)$~\citep{peng1994incremental}, and TB$(\lambda)$~\citep{precup2000eligibility} achieve a balance between return-based algorithms (where return refers to the sum of discounted rewards $\sum_{t} \gamma^{t} r_{t}$) and bootstrap algorithms (where return refers to $r_t + V(s_{t+1})$), speeding up the convergence of reinforcement learning algorithms. Therefore, we study the application of Peng's Q$(\lambda)$ for QMIX,

\input{fig/qlambda1}

\textbf{Interpretation.} Q networks without sufficient training usually have a large bias that impacts bootstrap returns. Figure~\ref{fig:qlambda1} shows that Q$(\lambda)$ allows for faster convergence in our experiments by reducing this bias. However, large values of $\lambda$ may lead to failed convergence due to the large variance. Figure \ref{fig:qlambda1} shows that when $\lambda$ is set to 0.9, it has a detrimental impact on the performance of QMIX.

\textbf{Recommendation.} Use Q($\lambda$) with a small value of $\lambda$.

%% file: fig/qlambda1.tex
\begin{figure}[h]
  \centering
  \includegraphics[height=2.6cm]{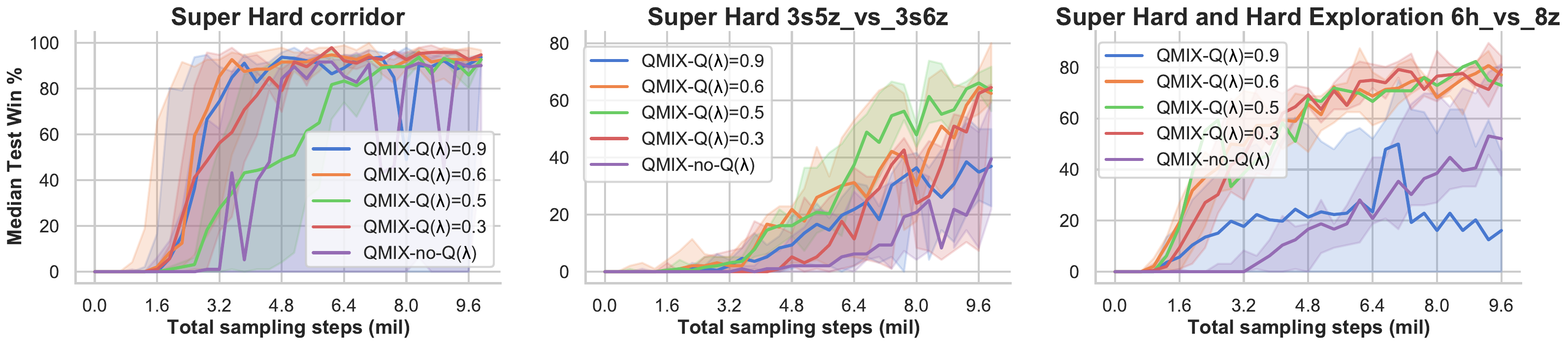}
  \caption{ Experiments for Q($\lambda$).}
\caption{Q($\lambda$) significantly improves performance of QMIX, but large values of $\lambda$ lead to instability in the algorithm.}
  \label{fig:qlambda1}
\end{figure}

%% file: sections/Experiment/ImplementationTricks/buffersize.tex
\subsubsection{Replay Buffer Size} \label{replay_size}

\textbf{Study description.} In single-agent Deep Q-networks (DQN), the replay buffer size is usually set to a large value. However, in multi-agent tasks, as the action space becomes larger than that of single-agent tasks, the distribution of samples changes more quickly. In this section, we study the impact of the replay buffer size on performance.

\input{fig/replay_buffer}

\textbf{Interpretation.} Figure~\ref{fig:replay_buffer} shows that a large replay buffer size causes instability in QMIX's learning. The causes of this phenomenon are as follows: (1) In multi-agent tasks, samples become obsolete more quickly than in single-agent tasks. (2)  Echoing in Sec.~\ref{optimization}, Adam performs better with samples with fast updates. (3) When the sampling policy is far from the current policy, the return-based methods require importance sampling ratios, which is difficult to calculate in multi-agent learning.

\textbf{Recommendation.} Use a small replay buffer size.

%% file: fig/replay_buffer.tex
\begin{figure}[h]
  \centering
  \includegraphics[height=2.6cm]{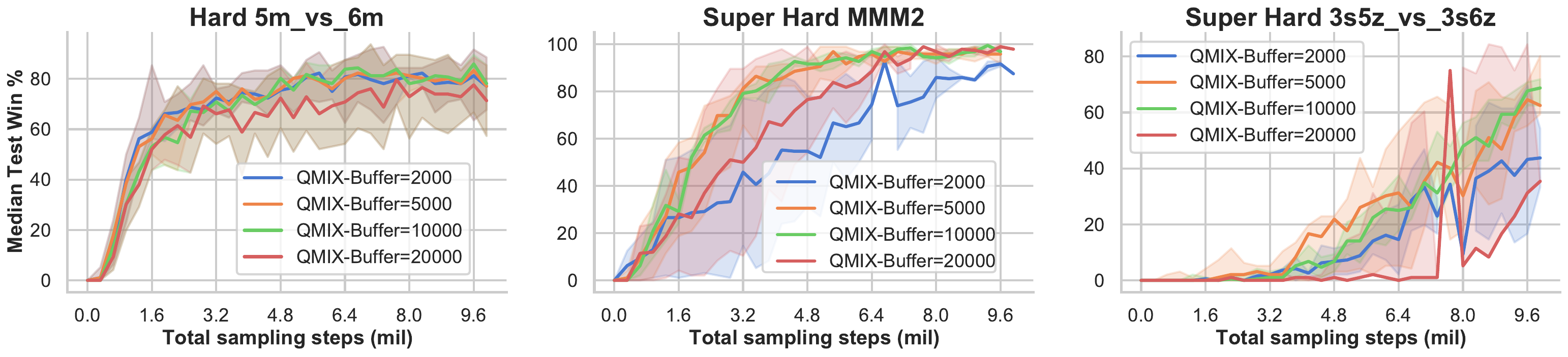}
\caption{Setting the replay buffer size to 5000 episodes allows for QMIX's learning to be more stable than by setting it to 20000 episodes.}
  \label{fig:replay_buffer}
\end{figure}

%% file: sections/Experiment/ImplementationTricks/processnumber.tex
\subsubsection{Rollout Process Number} 

\textbf{Study description.} When we collect samples in parallel as is done in A2C \citep{stooke2018accelerated}, it shows that when there is a defined total number of samples and an unspecified number of rollout processes, the median test performance becomes inconsistent. This study aims to perform analysis and provide insight on the impact of the number of processes on the final performance.

\input{fig/process_number}

\textbf{Interpretation.} Under the A2C \citep{mnih2016asynchronous} training paradigm, the total number of samples can be calculated as $S = E \cdot P \cdot I$, where
S is the total number of samples, E is the number of samples in each episode, P is the number of rollout processes, and I is the number of policy iterations. Figure \ref{fig:process1} shows that we are given both S and E; the fewer the number of rollout processes, the greater the number of policy iterations \citep{sutton2018reinforcement}; a higher number of policy iterations leads to an increase in performance. However, it also causes both longer training time and decreased stability.

\textbf{Recommendation.} Use fewer rollout processes when samples are difficult to obtain, especially for real-world robot learning; otherwise, use more rollout processes.

%% file: fig/process_number.tex
\begin{figure}[h]
  \centering
  \includegraphics[height=2.6cm]{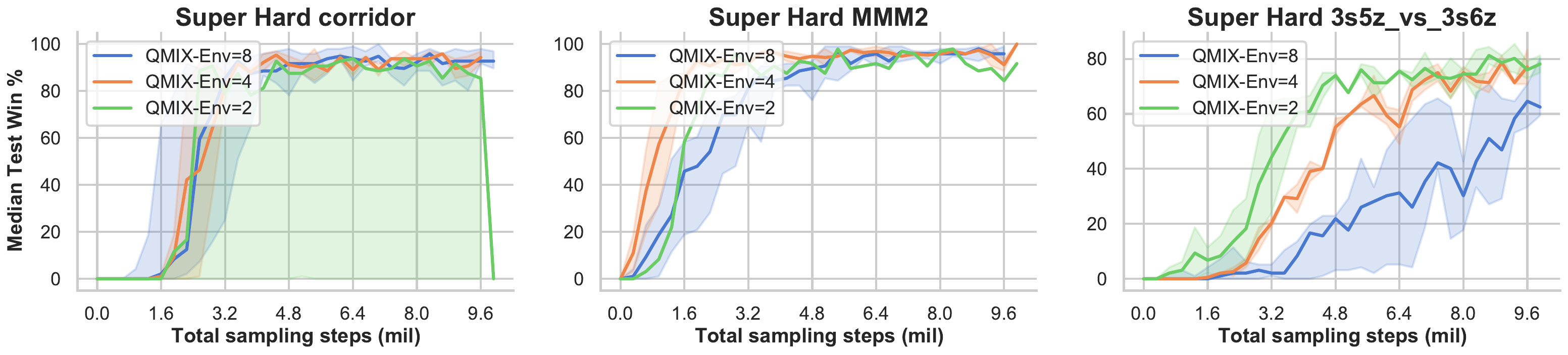}
\caption{Given the total number of samples, fewer processes achieve better performance. We set the replay buffer size to be proportional to the number of processes to ensure that the novelty of the samples is consistent.}
\label{fig:process1}
\end{figure}

%% file: sections/Experiment/ImplementationTricks/hidden_size.tex
\subsubsection{Hidden Size} \label{hidden_size}

\textbf{Study description.}  In deep reinforcement learning, researchers typically use smaller networks to train models, but the size of the neural network can also have an impact on algorithm performance. In this study, we analyze the network width of each layer of QMIX.

\input{fig/hidden_size}

\textbf{Interpretation.} As shown in Figure \ref{fig:hiddensize}, increasing the hidden size of the neural network of QMIX from 64 to 256 allows for a 18\% increase in win rate in the hard scenario $3s5z\_vs\_3s6z$. More specifically, increasing the hidden size of RNNs is more helpful to improve the performance of QMIX than increasing the hidden size of mixing networks. 

\textbf{Recommendation.} Increase the hidden size of QMIX to an appropriate value.

%% file: fig/hidden_size.tex
\begin{figure}[h]
  \centering
  \includegraphics[height=2.6cm]{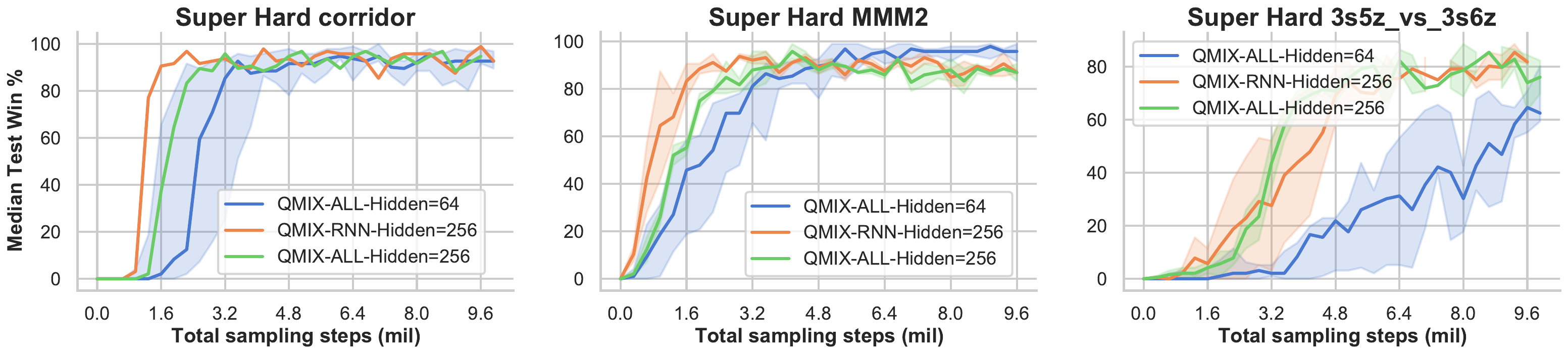}
\caption{On the hard scenario $3s5z\_vs\_3s6z$, increasing the width of neural network significantly improves the performance of QMIX.}
  \label{fig:hiddensize}
\end{figure}

%% file: sections/Experiment/ImplementationTricks/step.tex
\subsubsection{Exploration Steps} \label{exploration_steps}

\textbf{Study description.} Some scenarios in SMAC are hard to explore, such as $6h\_vs\_8z$, so the settings of $\epsilon$-greedy become critically important. In this study, we analyze the effect of $\epsilon$ anneal period on performance.

\input{fig/explore}

\textbf{Interpretation.} As shown in Figure \ref{fig:explore}, increasing the length of the $\epsilon$ anneal period from 100K steps to 500K steps allows for a increase in win rate in the Super Hard scenario $6h\_vs\_8z$ (38\%) and $3s5z\_vs\_3s6z$ (23\%). However, increasing this value to 1000K instead causes the model to collapse.

\textbf{Recommendation.} Increase the value of the $\epsilon$ anneal period to an appropriate length on hard-to-explore scenarios.

%% file: fig/explore.tex
\begin{figure}[h]
  \centering
  \includegraphics[height=2.6cm]{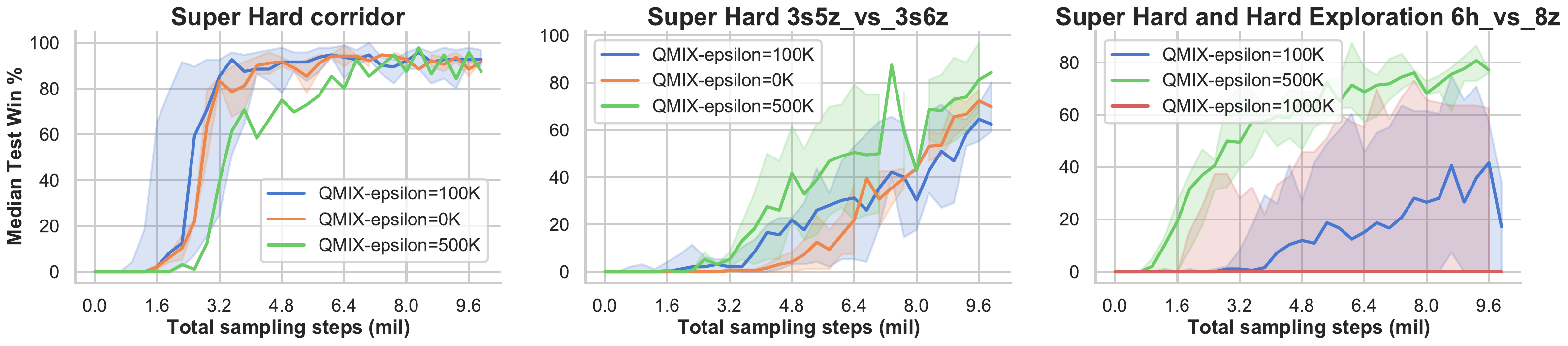}
\caption{On the hard-to-explore scenario $6h\_vs\_8z$, defining a proper length for $\epsilon$ anneal period significantly improves performance.}
  \label{fig:explore}
\end{figure}

%% file: sections/Experiment/ImplementationTricks/overall.tex
\subsubsection{Overall Impacts} \label{sec:overall}
\input{table/all_results}
Then we finetuned all these hyperparameters (besides the number of rollout processes) of QMIX for each scenarios of SMAC. As shown in Table \ref{table:all_results}, 
Finetuned-QMIX attains higher win rates in all hard and super hard SMAC scenarios, far exceeding vanilla QMIX. 

%% file: table/all_results.tex
\begin{table}[htbp] \footnotesize
\centering
\begin{tabular}{lccc}
\hline
Senarios       & Difficulty       &QMIX      & Finetuned-QMIX \\ \hline
2s\_vs\_1sc     & Easy  &\textbf{100\%}                & \textbf{100\%}        \\
2s3z     & Easy   &\textbf{100\%}              & \textbf{100\%}       \\
1c3s5z    & Easy    &\textbf{100\%}              & \textbf{100\%}     \\
3s5z & Easy        &\textbf{100\%}           & \textbf{100\%}    \\
10m\_vs\_11m   & Easy    &98\%              & \textbf{100\%}     \\
8m\_vs\_9m       & Hard    &84\%        & \textbf{100\%}     \\
5m\_vs\_6m     & Hard           & 84\%       & \textbf{90\%}     \\
3s\_vs\_5z     & Hard           & 96\%      & \textbf{100\%}     \\
bane\_vs\_bane & Hard           & \textbf{100\%}       & \textbf{100\%}    \\
2c\_vs\_64zg   & Hard           & \textbf{100\%}      & \textbf{100\%}     \\
corridor       & Super Hard     & 0\%       & \textbf{100\%}     \\
MMM2           & Super Hard     & 98\%       & \textbf{100\%}    \\
3s5z\_vs\_3s6z & Super Hard     & 3\%       & \textbf{93\%} (hidden\_size = 256)   \\
27m\_vs\_30m   & Super Hard     & 56\%       & \textbf{100\%}     \\
6h\_vs\_8z     & Super Hard     & 0\%       & \textbf{93\%} ($\lambda$ = 0.3)      \\ \hline
\end{tabular}
\caption{Best median test win rate of Finetuned-QMIX and QMIX (batch size=128) in all scenarios. 
}
\label{table:all_results}
\end{table}


%% file: sections/Experiment/Rethinking/main.tex
\subsection{Rethinking the Monotonicity Constraint} 
\label{rethink}
In this subsection, as the past studies evaluate the performance of QMIX's variants with inconsistent implementation tricks, we retested their performance based on the normalized tricks (details in Appendix~\ref{appendix:ExperimentsDetails}). In addition, RIIT and VMIX are demonstrated to further study the effects of the monotonicity constraint. 

\input{sections/Experiment/Rethinking/qmixtricks}

\input{sections/Experiment/Rethinking/riit}

%% file: sections/Experiment/Rethinking/qmixtricks.tex
\subsubsection{Re-Evaluation}
\label{section:QMIX with tricks}
\input{table/baseline}

We then normalize the tricks for all these algorithms for the re-evaluation, i.e, we perform grid search schemes on a typical hard environment (5m\_vs\_6m) and super hard environment (3s5z\_vs\_3s6z) to find \textbf{a general set of} hyperparameters for each algorithm (details in Appendix~\ref{appendix:hyperparameters}). As shown in Table~\ref{table:baselines}, the test results on the hardest scenarios in SMAC and DEPP demonstrate that, (1) The performance of values-based methods and VMIX with normalized tricks exceeds the test results in the past literatures \cite{samvelyan2019starcraft, wang2020qplex,peng2020facmac,rashid2020weighted, su2020valuedecomposition} (details in Appendix \ref{appendix:comparison}). (2) \textbf{QMIX outpeforms all its variants}. (3) The linear VDNs is also relatively effective. 
(4) The performance of the algorithm becomes progressively worse as the monotonicity constraint decreases ($\text{QMIX} > \text{QPLEX} > \text{WQMIX} > \text{LICA}$, details in Appendix \ref{appendix:Relationships}) in the benchmark environment.

The experimental results, specifically (2), (3) and (4), show that these variants of QMIX that relax the monotonicity constraint do not obtain better performance than QMIX in some purely cooperative tasks, either SMAC or DEPP.

%% file: table/baseline.tex
\begin{table}[htbp]
\centering
\resizebox{\textwidth}{24mm}{
\begin{tabular}{@{}llllccccccc@{}}
\toprule
\multirow{2}{*}{Scenarios} & \multirow{2}{*}{Difficulty} & \multicolumn{5}{c}{Value-based}                                                                  & \multicolumn{4}{c}{Policy-based}                          \\ \cmidrule(l){3-11} 
                           &                             & QMIX            & VDNs           & Qatten          & QPLEX          & \multicolumn{1}{c|}{WQMIX} & LICA           & VMIX   & DOP            & RIIT            \\ \midrule
2c\_vs\_64zg               & Hard                        & \textbf{100\%}  & \textbf{100\%} & \textbf{100\%}  & \textbf{100}\% & \textbf{100}\%             & \textbf{100}\% & 98\%   & 84\%           & \textbf{100}\% \\
8m\_vs\_9m                 & Hard                        & \textbf{100\%}  & \textbf{100\%} & \textbf{100\%}  & 95\%           & 95\%                       & 48\%           & 75\%   & 96\%           & 95\%           \\
3s\_vs\_5z                 & Hard                        & \textbf{100\%}  & \textbf{100\%} & \textbf{100} \% & \textbf{100}\% & \textbf{100}\%             & 3\%            & 96\%   & \textbf{100}\% & 96\%           \\
5m\_vs\_6m                 & Hard                        & \textbf{90\%}   & \textbf{90\%}  & \textbf{90\%}   & \textbf{90\%}  & \textbf{90\%}              & 53\%           & 9\%    & 63\%           & 67\%           \\
3s5z\_vs\_3s6z             & S-Hard                      & \textbf{75\%}   & 43\%           & 62\%            & 68\%           & 56\%                       & 0\%            & 56\%   & 0\%            & \textbf{75\%}  \\
corridor                   & S-Hard                      & \textbf{100\%}  & 98\%           & \textbf{100\%}  & 96\%           & 96\%                       & 0\%            & 0\%    & 0\%            & \textbf{100\%} \\
6h\_vs\_8z                 & S-Hard                      & 84\%            & \textbf{87\%}  & 82\%            & 78\%           & 75\%                       & 4\%            & 80\%   & 0\%            & 19\%           \\
MMM2                       & S-Hard                      & \textbf{100\%}  & 96\%           & \textbf{100\%}  & \textbf{100\%} & 96\%                       & 0\%            & 70\%   & 3\%            & \textbf{100}\% \\
27m\_vs\_30m               & S-Hard                      & \textbf{100\%}  & \textbf{100\%} & \textbf{100\%}  & \textbf{100\%} & \textbf{100}\%             & 9\%            & 93\%   & 0\%            & 93\%           \\
Discrete PP                & -                           & \textbf{40}     & 39             & -               & 39             & 39                         & 30             & 39     & 38             & 38             \\
Avg. Score                 & \textbf(Hard+)              & \textbf{94.9\%} & 91.2\%         & 92.7\%          & 92.5\%         & 90.5\%                     & 29.2\%         & 67.4\% & 44.1\%         & 84.0\%         \\ \bottomrule
\end{tabular}
}\caption{Median test winning rate (episode return) of MARL algorithms with normalized tricks. S-Hard denotes Super Hard. We compare their performance in the most difficult scenarios of SMAC and the Discrete PP.}
\label{table:baselines}
\end{table}

%% file: sections/Experiment/Rethinking/riit.tex
\subsubsection{Albation Studies of Monotonicity Constraint} 
\label{phasic}
\input{fig/riit}

We further study the impact of monotonicity constraint tasks via comparing the performance of adding or removing the constraint. An end-to-end Actor-Critic method, RIIT, is proposed. Specifically, we use the monotonic mixing network as a critic network, shown in Figure \ref{fig:qmix}. Then, in Eq.~\ref{eqn:riit}, with a trained critic $Q_{\theta_{c}}^{\pi}$ estimate, the decentralized policy networks $\pi_{\theta_{i}}^{i}$ can then be optimized end-to-end simultaneously by maximizing $Q_{\theta_{c}}^{\pi}$ with the policies $\pi_{\theta_{i}}^{i}$ as inputs. Since RIIT is trained end-to-end, it may also be used for continuous control tasks. It is worth stating that the item  $\mathbb{E}_{i}\left[\mathcal{H}\left(\pi_{\theta_{i}}^{i}\left(\cdot \mid z_{t}^{i}\right)\right)\right]]$ is the Adaptive Entropy \cite{zhou2020learning}, and we use a two-stage approach to train the actor-critic network, described in detail in Appendix \ref{riit_train}. 
\begin{eqnarray}
\begin{aligned}
\max _{\theta} \mathbb{E}_{t, s_{t}, u_{t}^{1}, \ldots, \tau_{t}^{n}}[Q_{\theta_{c}}^{\pi}\left(s_{t}, \pi_{\theta_{1}}^{1}\left(\cdot \mid \tau_{t}^{1}\right), \ldots, \pi_{\theta_{n}}^{n}\left(\cdot \mid \tau_{t}^{n}\right)\right)
+ \mathbb{E}_{i}\left[\mathcal{H}\left(\pi_{\theta_{i}}^{i}\left(\cdot \mid \tau_{t}^{i}\right)\right)\right]]
\end{aligned} 
\label{eqn:riit}
\end{eqnarray}

The monotonicity constraint on the critic (Figure \ref{fig:qmix}) is theoretically no longer required as the critic is not used for greedy action selection. We can evaluate the effects of the monotonicity constraint by removing the absolute value operation in the monotonic mixing network. In this way, RIIT can also be easily extended to non-monotonic tasks. Figure~\ref{fig:riit_abla}  demonstrates that the monotonicity constraint significantly improves the performance of RIIT. Table \ref{table:baselines} also presents that RIIT performs best among these policy-based algorithms. 

\input{fig/vmix2} 

To explore the generality of monotonicity constraints, we extend the above experiments to VMIX ~\cite{su2020valuedecomposition}. VMIX adds the monotonicity constraint to the value network (not Q value networks) of A2C. (details in Appendix~\ref{appendix:vmix}) VMIX learns the policy of each agent by advantage-based policy gradient \cite{mnih2016asynchronous}; therefore, the monotonicity constraint is not necessary for greedy action selection either. We can evaluate the effects of the monotonicity constraint by removing the absolute value operation in Figure \ref{fig:vmix_net}. The result from Figure~\ref{fig:vmix} shows that the monotonicity constraint improves the sample efficiency in value networks. The above experimental results indicate that the monotonicity constraint can improve the sample efficiency in some multi-robot cooperative tasks, such as SMAC and DEPP.

%% file: fig/riit.tex
\begin{figure}[h]
  \centering
  \includegraphics[height=2.5cm]{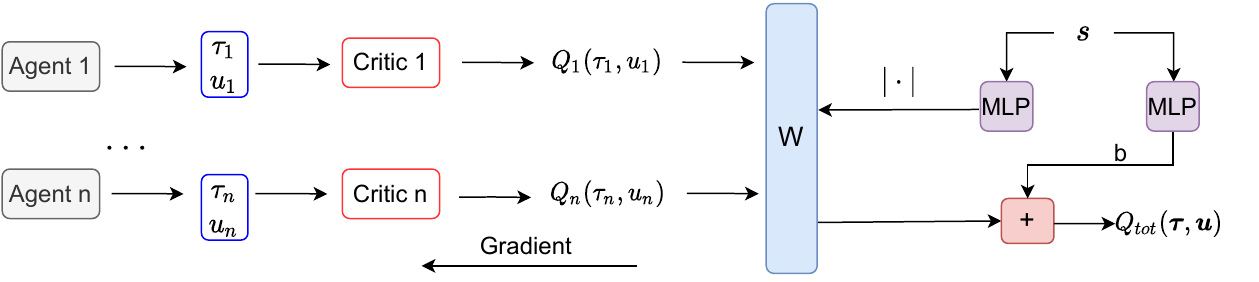}
    \caption{Architecture for RIIT: $|\cdot| $ \textbf{denotes absolute value operation}, implementing the monotonicity constraint of QMIX. $W$ denotes the non-negative mixing weights. Agent $i$ denotes the policy network which can be trained end-to-end by maximizing the $Q_{tot}$.}
  \label{fig:qmix}
\end{figure}

\begin{figure}[h]
\centering
\includegraphics[height=4.7cm]{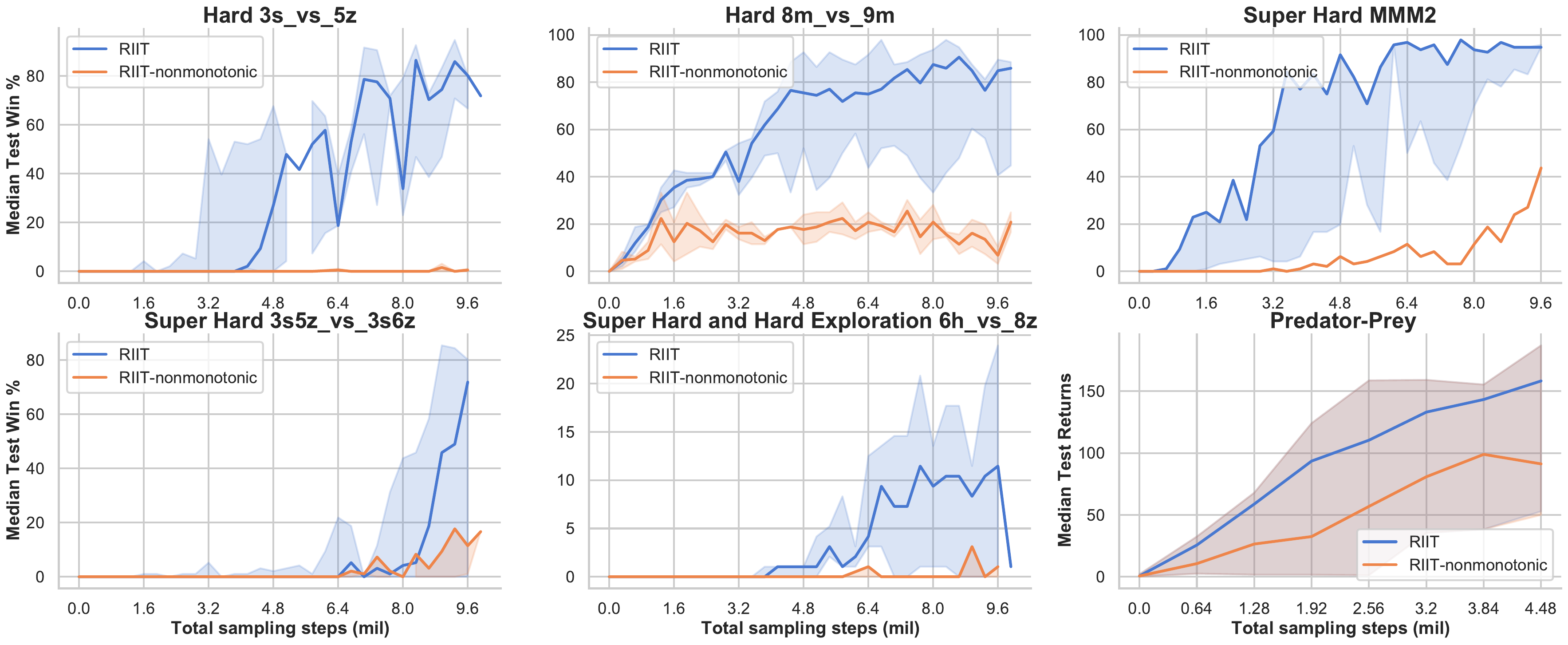}
\caption{Comparing RIIT w./ and w./o. monotonicity constraint (remove absolute value operation) on SMAC and Continuous Predator-Prey.}
\label{fig:riit_abla}
\end{figure}

%% file: fig/vmix2.tex
\begin{figure}[h]
\centering
\includegraphics[height=2.6cm]{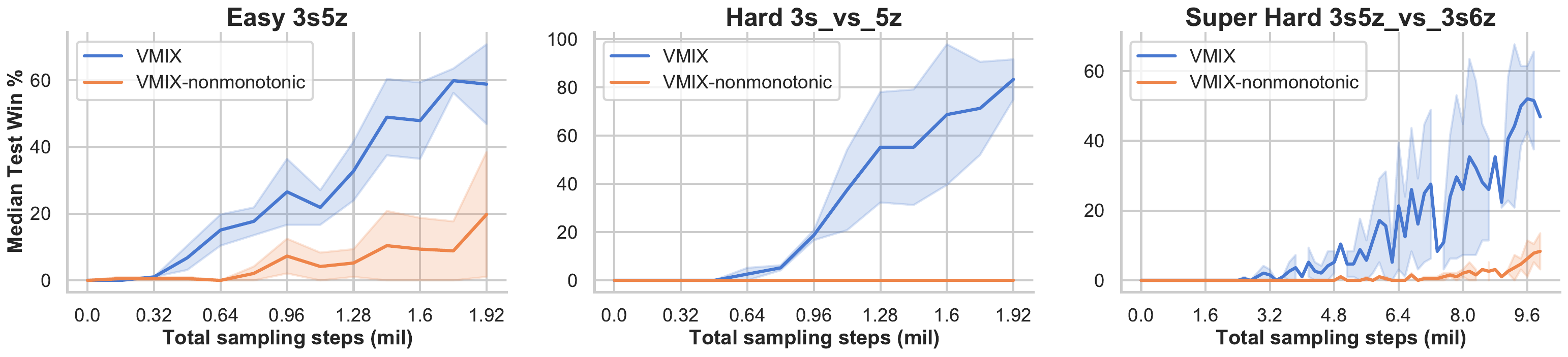}

\caption{Comparing VMIX with and without monotonicity constraint on SMAC.}
\label{fig:vmix}
\end{figure}

%% file: sections/discussion.tex
\section{Discussion} 
\label{theory}
To better understand the monotonicity constraint, we discuss the following two questions with theoretical analysis. \textbf{Ques.1} Why can SMAC be represented well by monotonic mixing networks? \textbf{Ques.2} Why can the monotonicity constraint improve the sample efficiency in SMAC?
To coherently answer the above questions, we give the following definitions and propositions. It is worth noting that the core assumption is that the joint action-value function $Q_{tot}$ can be represented by a non-linear mapping $f_\phi(s; Q_1, Q_2, ... Q_N)$, \textit{but without the monotonicity constraint}.

\input{discussion/coo}
\input{discussion/hu}
\input{discussion/competitive}
\input{discussion/purely}
\input{discussion/proposition}


\textbf{For answering Ques.1}: According to the Proposition \ref{proposition}, we need to explain why SMAC can be seen as a purely cooperative task environment. SMAC mainly uses a shaped reward signal calculated from the hit-point damage dealt, some positive reward after having enemy units killed and a positive bonus for winning the battle. In practice, we can decompose the hit-point damage dealt linearly, and divide the units killed rewards to the agents near the enemy evenly, the victory rewards to all agents. \textbf{This fair reward decomposition can be interpreted as purely cooperative.} Futuremore, as $Q_{\pi}(s, u) =\mathbb{E}_{\pi}[\sum_{k=0}^{\infty} \gamma^{k} r_{t+k+1} \mid s, u]$, the reward is linearly assignable meaning that Q value is linearly assignable, which also explains why the VDNs also work well in SMAC (Table.~\ref{table:baselines}). 

\textbf{For answering Ques.2}: Just as in RIIT's implementation (Figure \ref{fig:qmix}) where the monotonicity constraint reduces the range of values of each mixing weight by half, the hypothesis space is assumed to decrease exponentially by $(\frac{1}{2})^N$ (N denotes the number of weights). Note that the Q value decomposition mapping of the SMAC is a subset of the hypothesis space of QMIX's mixing network. Therefore, using the monotonicity constraint can allow for avoiding searching invalid parameters, leading to a significant improvement in sampling efficiency.

Our analysis shows that QMIX works well if a multi-agent task can be interpreted as purely cooperative, even if it can also be interpreted as competitive. \textbf{That is, QMIX will try to find a purely cooperative interpretation for a complex multi-agent task.}

%% file: discussion/coo.tex
\newtheorem{definition}{Definition}
\begin{definition} \label{def:coo}
\textbf{Cooperative tasks.} For a task with N agents ($N > 1$), all agents have a common goal.
\end{definition}

%% file: discussion/hu.tex
\begin{definition} \label{def:hu}
\textbf{Semi-cooperative Tasks.}
 Given a cooperative task with a set of agents $\mathbb{N}$. For all states $s$ of the task, if there is a subset $\mathbb{K} \subseteq \mathbb{N}$, $\mathbb{K} \neq \varnothing$, where the $Q_i, i \in \mathbb{K}$ increases while the other $Q_j, j \notin \mathbb{K}$ are fixed, this will lead to an increase in $Q_{tot}$. 
\end{definition}

As a counterexample, the collective action problem (social dilemma) is not Semi-cooperative task. i.e., since the Q value may not include future rewards when $\gamma <$  1, the collective interest in the present may be detrimental to the future interest.

%% file: discussion/competitive.tex
\begin{definition} \label{def:competitive}
\textbf{Competitive Cases.} Given two agents $i$ and $j$, we say that agents $i$ and $j$ are competitive if either an increase in $Q_i$ leads to a decrease in $Q_j$ or an increase in $Q_j$ leads to a decrease in $Q_i$.
\end{definition}

%% file: discussion/purely.tex
\begin{definition} \label{purely}
\textbf{Purely Cooperative Tasks.} Semi-cooperative tasks without competitive cases.
\end{definition}

As an counterexample, the matrix game as in Table \ref{nonmatrix} is not a purely cooperative task. Because of the random action sampling in reinforcement learning, we cannot guarantee that the agents share the same preferences. If one agent prefers action 0 (Like hunting) and the other agent prefers action 1 or 2 (Like sleeping or entertaining), they will have a conflict of interest (Those who like to sleep will cause the hunter to fail to catch the prey).

%% file: discussion/proposition.tex
\newtheorem{theorem}{Proposition}
\begin{theorem} \label{proposition}
 Purely Cooperative Tasks can be represented by monotonic mixing networks.
\end{theorem}

\begin{proof}
Since the QMIX's mixing network is a universal function approximator of monotonic functions, for a Semi cooperative task, if there is a case (state $s$) that cannot be represented by a monotonic mixing network, i.e., $\frac{\partial Q_{tot}(s)}{\partial Q_i} < 0$, then an increase in $Q_i$ must lead to a decrease in $Q_j, j \neq i$ (since there is no $Q_j$ decrease, by Def. \ref{def:hu}, the constraint $\frac{\partial Q_{tot}(s)}{\partial Q_i} < 0$ does not hold). Therefore, by Def. \ref{def:competitive} this cooperative task has a competitive case which means it is not a purely cooperative task.
\end{proof}

%% file: sections/conclusion.tex
\section{Conclusion} \label{conclu}
In this paper, we investigate the influence of certain code-level optimizations on the performance of QMIX and provide tuning optimizations suggestions. Then, we find that monotonicity constraint can improve sample efficiency in SMAC and DEPP, benefiting to the real-world robot learning. Our analysis imply that we can design reward functions in the real multi-agent task that can be interpreted as purely cooperative, improving the learning sample efficiency of the MARL. Meanwhile, we also believe that the variants that relax monotonicity constraint of QMIX might be well-suited for the mutil-agent tasks which cannot be interpreted as purely cooperative. In addition, we are hopeful that this paper will call on the community to be more fair in comparing the performance of algorithms.

%% file: appendix/main.tex
\input{appendix/basic_tricks}
\input{appendix/tricks}
\input{appendix/finetune}

\input{appendix/comparison}
\input{appendix/RIIT}
\input{appendix/supplementary}

%% file: appendix/basic_tricks.tex
\section{Rethinking the Code-level
\label{appendix:basic_tricks}
Optimizations (Extension of Sec.~\ref{experi})}
Engstrom~\textit{et.al} \cite{engstrom2020implementation} investigates code-level optimizations based on PPO \cite{schulman2017proximal} implementation, and concludes that the majority of performance differences between PPO and TRPO originate from code-level optimizations. Andrychowicz \textit{et. al}~\cite{andrychowicz2020what} investigates the influence of code-level optimizations on the performance of PPO and provides tuning optimizations. These optimizations include: (1) Adam and Learning rate annealing. (2) Orthogonal initialization and Layer scaling. (3) Observation normalization. (4) Value normalization. (5) N-step returns (eligibility traces). (6) Reward scaling. (7) Reward clipping. (8) Neural Network Size. Using a subset of the whole code-level optimizations, specifically shown in Sec.~\ref{experi}, we enabled QMIX to solve almost all scenarios of SMAC.
 
We also propose a simple trick, i.e, rewards shaping, to help QMIX learning in a non-monotonic environment. 

\input{sections/Experiment/ImplementationTricks/rewards}

%% file: sections/Experiment/ImplementationTricks/rewards.tex
\subsection{Rewards Shaping}
\textbf{Study description.}
Table \ref{nonmatrix} shows a non-monotonic case that QMIX cannot solve. However, the reward function in MARL is defined by the user; we investigate whether QMIX can learn a correct argmax action by reshaping the task's reward function without changing its goal.

\input{table/nonmatrix}

\textbf{Interpretation.} 
The reward -$12$ in Table \ref{nonmatrix} does not assist the agents in finding the optimal solution; as shown in Table \ref{nonmatrix3}, this non-monotonic matrix may be solved by simply replacing the insignificant reward -$12$ with -$0.5$. The reward shaping may also help QMIX learn more effectively in other non-monotonic tasks. 

\textbf{Recommendation.} 
Increase the scale of the important rewards of the tasks and reduce the scale of rewards that may cause disruption.

%% file: table/nonmatrix.tex
\begin{table}[hbtp]
 \begin{minipage}{0.48\columnwidth}
  \centering
      \begin{tabular}{|p{.7cm}|p{.7cm}|p{.7cm}|}
        \hline \ \textbf{12.0} & -0.5 & -0.5 \\
        \hline -0.5 & 0 & 0 \\
        \hline -0.5 & 0 & 0 \\
        \hline
    \end{tabular}
        \subcaption{Reshaped Payoff matrix}
  \end{minipage}
  \begin{minipage}{0.48\columnwidth}
  \centering
         \begin{tabular}{|p{.7cm}|p{.7cm}|p{.7cm}|}
        \hline \textbf{12.0} & -0.3 & -0.3 \\
        \hline -0.3 & -0.3 & -0.3 \\
        \hline -0.3 & -0.3 & -0.3 \\
        \hline
        \end{tabular}
        \subcaption{QMIX: $Q_{tot}$ }
  \end{minipage}
  \caption{A non-monotonic matrix game in which we reshape the reward by replacing the insignificant reward -$12$ (in Table \ref{nonmatrix}) with reward -$0.5$. QMIX learns a $Q_{tot}$ which has a correct argmax. Bold text indicates argmax action's reward.}
  \label{nonmatrix3}
\end{table}

%% file: appendix/tricks.tex
\subsection{Peng's Q($\lambda$)} 
\label{qlambda}
We briefly introduce Peng's Q($\lambda$) here,
TD$(\lambda)$ can be expressed as Eq. \ref{eqn:td}:

\begin{eqnarray} \label{eqn:td}
\begin{aligned}
&G_{s}^{\lambda} \doteq(1-\lambda) \sum_{n=1}^{\infty} \lambda^{n-1} G_{s: s+n} \\
&G_{s: s+n} \doteq \sum_{t=s}^{s+n} \gamma^{t-s} r_{t}+\gamma^{n+1} V\left(s_{s+n+1}, u\right)
\end{aligned}
\end{eqnarray}

Peng's Q$(\lambda)$ replaces the V value of the next state with the max Q value, as shown in Eq. \ref{eqn:qlambda}:

\begin{eqnarray} \label{eqn:qlambda}
\begin{aligned}
&G_{s: s+n} \doteq \sum_{t=s}^{s+n} \gamma^{t-s} r_{t}+\gamma^{n+1} \max _{u} Q\left(s_{s+n+1}, u\right)
\end{aligned}
\end{eqnarray}

where $\lambda$ is the discount factor of the traces and $\left(\prod_{s=1}^{t} \lambda\right)=1 \text { when } t=0$. When $\lambda$ is set to 0, it is equivalent to 1-step bootstrap returns. When $\lambda$ is set to 1, it is equivalent to Monte Carlo \cite{sutton2018reinforcement} returns. \cite{kozuno2021revisiting} show that while Peng's Q$(\lambda)$ does not learn optimal policies under arbitrary behavior policies, a convergence guarantee can be recovered if the behavior policy tracks the target policy, as is often the case in practice.

%% file: appendix/finetune.tex
\section{Experimental Details}
\label{appendix:ExperimentsDetails}
\subsection{Hyperparameters} 
\label{appendix:hyperparameters}
\begin{table}[htbp] \footnotesize
\centering
\begin{tabular}{lcc}
\hline
Tricks              & Value-based (VB)       & Policy-bassed  (PG)                   \\ \hline
Optimizer               & Adam, RMSProp          & Adam, RMSProp                         \\
Learning Rates          & 0.0005, 0.001          & 0.0005, 0.001, (and 0.0001 for DOP)          \\
Batch Size(episodes)    & 32, 64, 128            & 32, 64                           \\
Replay Buffer Size      & 5000, 10000, 20000     & 2000, 5000, 10000, 20000              \\
Q($\lambda$), TD($\lambda$)            & 0, 0.3, 0.6, 0.9          & 0, 0.3, 0.6, 0.9                         \\
(Adaptive) Entropy        & -                      & 0.01, 0.03, 0.06 \\
$\epsilon$ Anneal Steps & 50K, 100K, 500K, 1000K &  100K, 500K for DOP\\   \hline                                 
\end{tabular}
\caption{\centering{Hyperparameters Search on SMAC.}}
\label{table:search}
\end{table}

In this section, we present our tuning process. We get the optimal hyperparameters for each algorithm by the grid search, shown in Table \ref{table:search}. Specifically,
\begin{enumerate}
    \item For experiments in Sec. \ref{section:QMIX with tricks}, we perform grid search schemes on a typical hard environment (5m\_vs\_6m) and super hard environment (3s5z\_vs\_3s6z) to find a \textbf{general} set of hyperparameters for each algorithm. In this way, we can evaluate the robustness of these MARL algorithms. 
    \item For experiments in Sec. \ref{sec:overall}, we perform hyperparameter search on each scenarios for QMIX to demonstrate the best performance of QMIX.
\end{enumerate}

\begin{table}[htbp] \footnotesize
\begin{subtable}[]{1.0\textwidth}
\centering
\begin{tabular}{lccccc}
\hline
Algorithms              & LICA   & \textbf{OurLICA} & DOP                     & \textbf{OurDOP}         & RIIT           \\ \hline
Optimizer               & Adam   & \textbf{Adam}    & RMSProp                 & RMSProp                 & Adam          \\
Batch Size(episodes)    & 32     & 32               & Off=32, On=16           & \textbf{Off=64, On=32}  & Off=64, On=32 \\
TD($\lambda$)           & 0.8    & \textbf{0.6}     & 0.8, TB($\lambda$=0.93) & 0.8, TB($\lambda$=0.93) & 0.6           \\
Adaptive Entropy        & 0.06   & 0.06             & -                       & -                       & 0.03          \\
$\epsilon$ Anneal Steps & -      & -                & 500K (double)     & 500K (double)     & -             \\
Critic-Net Size         & 29696K & \textbf{389K}    & 122K                    & 122K                    & 69K           \\
Rollout Processes       & 32     & \textbf{8}       & 4                       & \textbf{8}              & 8             \\ \hline
\end{tabular}
\caption{\centering{Setting of Policy-based algorithms.}} {double: DOP first adds noise to the output of the policy network, then mask invalid actions and adds noise to the probabilities again.}
\label{table:policy}
\end{subtable}


\begin{subtable}[]{1.0\textwidth}
\centering
\begin{tabular}{lcccccccc}
\hline
Algorithms           & QMIX     & \textbf{OurQMIX} & Qatten  & \textbf{OurQatten} & QPLEX   & \textbf{OurQPLEX}  \\ \hline
Optimizer            & RMSProp & \textbf{Adam}     & RMSProp & \textbf{Adam}       & RMSProp & \textbf{Adam}      &      \\
Batch Size (epi.)           & 128      & \textbf{128}      & 32      & \textbf{128}        & 32      & \textbf{128}         \\
Q($\lambda$)     & 0       & \textbf{0.6}      & 0       & \textbf{0.6}        & 0       & \textbf{0.6}              \\
Attention Heads      & -       & -                 & 4       & 4                   & 10      & \textbf{4}     \\
\makecell{Mixing-Net Size} & 41K     & 41K               & 58K     & 58K                 & 476K    & \textbf{152K}               \\
$\epsilon$ Anneal Steps                                                                    & \multicolumn{6}{c}{50K $\rightarrow$ \textbf{500K for $6h\_vs\_8z$, 100 K for others}}                                                                                                                                                     \\
Rollout Processes   & 8       & \textbf{8}        & 1       & \textbf{8}          & 1       & \textbf{8}                \\ \hline
\end{tabular}
\caption{\centering{Setting of Value-based algorithm.}}
\label{table:value}
\end{subtable}
\caption{\centering{Hyperparameters Settings.}}
\end{table}

Table \ref{table:policy} and \ref{table:value} shows our general settings for the these algorithms. The network size is calculated under $6h\_vs\_8z$, where adding \textit{Our} denotes the new hyperparameter settings. Next, we describe in detail the setting of these hyperparameters,

\textbf{Neural Network Size} We first ensure the network size is the same order of magnitude, which means that we decrease the critic-net size of LICA from 29696K to 389K, and we use 4 attention heads leading the mixing-net size of QPLEX from 476K to 152K. All the agent networks are the same as those found in QMIX \cite{rashid2018qmix}.

\textbf{Optimizer \& Learning Rate}
We use Adam to optimize all networks, except VMIX (works better with RMSProp), as it may accelerate the convergence of the algorithms. Furthermore, we use different learning rates for each algorithm: (1) For all value-based algorithms, neural networks are trained with 0.001 learning rate. (2) For LICA, we set the learning rate of the agent network to 0.0025 and the critic network's learning rate to 0.0005. (3) For RIIT and VMIX, we set the learning rates to 0.001. 

\textbf{Batch Size}
We find that a large batch size helps to improve the stability of the algorithms. Therefore, for value-based algorithms, we set the batch size to 128. For the policy-based algorithms, we set the batch size to 64/32 (Offline/Online training) due to the fact that online update requires only the newest data.

\textbf{Replay Buffer Size} As discussed in Appendix. \ref{replay_size}, a small replay buffer size facilitates the convergence of the MARL algorithms. Therefore, for SMAC, the size of all replay buffers is set to 5000 episodes. For Predator-Prey, we set the buffer size to 1000 episodes.

\textbf{Exploration}
 As discussed in Appendix. \ref{exploration_steps}, we use $\epsilon$-greedy action selection, decreasing $\epsilon$ from 1 to 0.05 over n-time steps (n can be found in Table \ref{table:value}) for value-based algorithms. We use the Adaptive Entropy \cite{zhou2020learning} (Appendix. \ref{lica}) for all policy-based algorithms, except VMIX (works better with ordinary entropy and annealing noise), because it facilitates the automatic adjustment of the size of the entropy loss in different scenarios. Specialy, we add the Adaptive Entropy to DOP to prevent it from crashing in SMAC.
 
\textbf{N-step returns} 
We find that the $\lambda$ values of Q($\lambda$) and TD($\lambda$) are hevily depend on the scenario. We are using $\lambda$ = 0.6 for all tasks as the value works stably in most scenarios. However, for the on-policy method VMIX, we set $\lambda$ = 0.8. 

\textbf{Rollout Processes Number} For SMAC and Discrete PP, 8 rollout processes for parallel sampling are used to obtain as many samples as possible from the environments at a high rate. This also ensures that all the algorithms share the same number of policy iterations and sample size (10 million). For the non-monotonic matrix games, we set the processes number to 32. At last, 4 rollout processesare used for Continuous PP.

\textbf{Other Settings}
We set all discount factors $\gamma$ = 0.99. We update the target network every 200 episodes. We find that the optimal hyperparameters of the value-based algorithms are similar due to the fact that they share the same basic architecture and training paradigm. Therefore, the settings for VDNs and WQMIX are the same as for QMIX. Specifically, we use OW-QMIX, detailed in~\ref{appendix:wqmix}, in WQMIX as the baseline. 

Note that our experimental results are not directly comparable with the previous works (which use SC2.4.6), as we use StarCraft 2 (SC2.4.10) in the latest PyMARL.

%% file: appendix/comparison.tex
\subsection{The Performance of Original Algorithms}
\label{appendix:comparison}
In this section, we compare performance of the original algorithms with third-party experimental results, i.e. experimental results of the paper citing the algorithm.

For VDNs and QMIX, the original SMAC paper ~\cite{samvelyan2019starcraft} shows that VDNs and QMIX do not perform well in hard and super hard scenarios. For Qatten, the experiments in ~\cite{wang2020qplex} demonstrates that the performance of Qatten is worse than vanilla QMIX. ~\cite{peng2020facmac} demonstrates that QPLEX and DOP does not work well in hard and super hard scenarios in SMAC, and the their performance is worse than vanilla QMIX. It is interesting that WQMIX~\cite{rashid2020weighted} shows the poor performance of WQMIX in super hard scenarios $3s5z\_vs\_3s6z$ and $corridor$. The original test results in LICA are not considered as 64 million samples are used in their experiments.

However, after our hyperparameter tuning, all the value-based methods perform well in Hard and Super Hard scenarios. This shows that our hyperparameters does improve their performance.

%% file: appendix/riit.tex
\section{RIIT} \label{riit_train}
In this section, we show the pseudo-code for the training procedure of RIIT. (1) Training the critic network with  offline samples and 1-step TD error loss improves the sample efficiency for critic networks; (2) Training policy networks end-to-end and critic with TD($\lambda$) and online samples improves learning stability of RIIT ~\footnote{\cite{cobbe2020phasic} shows that actor-networks generally have a lower tolerance for sample reuse than critic networks; and for RIIT, our empirical evidence shows that TD($\lambda$) is not stable in the offline samples.}.

\begin{algorithm}[ht] 
\caption{Optimization Procedure for RIIT}
\label{alg:training-procedure}
\begin{algorithmic}
\State Initialize offline replay memory $D$ and online replay memory $D'$.
\State Randomly initialize $\theta$ and $\phi$ for the policy networks and the mixing critic respectively.
\State Set $\phi^- \leftarrow \phi$.
\While{not terminated}
\State \parbox[t]{\dimexpr\textwidth-\leftmargin-\labelsep-\labelwidth}{%
\# Off-policy stage\\
Sample $b$ episodes $\tau_{1}, ..., \tau_{b}$ with $\tau_{i} = \{s_{0,i}, o_{0,i}, u_{0,i}, r_{0,i}, ..., s_{T,i}, o_{T,i}, u_{T,i}, r_{T,i}\}$ from offline replay memory $D$.
\strut}
\State \parbox[t]{\dimexpr\textwidth-\leftmargin-\labelsep-\labelwidth}{%
Update the monotonic mixing network with $y_{t,i}$ calculated by 1-step bootstrap return ($y_{t,i} = r_{t,i} + \gamma Q_{\phi^-}^{\pi}(s_{t+1}, \Vec{u}_{t+1})$): \strut}
\begin{equation} \label{eq:supp-critic-update}
    \nabla_{\phi}  \frac{1}{bT} \sum_{i=1}^{b}\sum_{t=1}^{T}\left( y_{t,i} - Q_{\phi}^{\pi}\left(s_{t,i}, u_{t,i}^1, ..., u_{t,i}^n\right)\right)^2.
\end{equation}
\State \parbox[t]{\dimexpr\textwidth-\leftmargin-\labelsep-\labelwidth}{%

\# On-policy stage\\
Sample $b$ episodes $\tau_{1}, ..., \tau_{b}$ with $\tau_{i} = \{s_{0,i}, o_{0,i}, u_{0,i}, r_{0,i}, ..., s_{T,i}, o_{T,i}, u_{T,i}, r_{T,i}\}$ from online replay memory $D'$.
\strut}
\State \parbox[t]{\dimexpr\textwidth-\leftmargin-\labelsep-\labelwidth}{%
Update the monotonic mixing network with $y_{t,i}^{TD(\lambda)}$ calculated by TD($\lambda$): \strut}
\begin{equation} \label{eq:supp-critic-update}
    \nabla_{\phi}  \frac{1}{bT} \sum_{i=1}^{b}\sum_{t=1}^{T}\left(y_{t,i}^{TD(\lambda)} - Q_{\phi}^{\pi}\left(s_{t,i}, u_{t,i}^1, ..., u_{t,i}^n\right)\right)^2.
\end{equation}
\State \parbox[t]{\dimexpr\textwidth-\leftmargin-\labelsep-\labelwidth}{%
Update the decentralized policy networks end-to-end by maximizing the Q value, with Adaptive Entropy Loss (the entropy normalized by its module length) :\strut}
\begin{equation}
    \nabla_{\theta}  \frac{1}{bT} \sum_{i=1}^{b} \sum_{t=1}^{T} \left( -Q_{\phi}^{\pi}\left(s_{t,i}, \pi^1_{\theta_1}(\cdot | z^1_{t,i}), ..., \pi^n_{\theta_n}(\cdot | z^n_{t,i})\right) - \frac{1}{n} \sum_{a=1}^{n} \mathcal{H}\left(\pi^a_{\theta_a}(\cdot|z^a_{t,i})\right) \right).
\end{equation}
\If{at target update interval}
    \State \parbox[t]{\dimexpr\textwidth-\leftmargin-\labelsep-\labelwidth}{%
    Update the target mixing network $\phi^- \leftarrow \phi$.
    \strut}
\EndIf
\EndWhile
\end{algorithmic}
\end{algorithm}

%% file: appendix/supplementary.tex
\section{CTDE algorithms} 
\label{supplementary materials}
\subsection{IQL}
Independent Q-learning (IQL) \cite{tan1993multi} breaks down a multi-agent task into a series of simultaneous single-agent tasks that share the same environment, just like multi-agent Deep Q-networks (DQN) \cite{mnih2013playing}. DQN represents the action-value function with a deep neural network parameterized by $\theta$. DQN uses a replay buffer to store transition tuple $\left\langle s, u, r, s^{\prime}\right\rangle$, where state $s^{\prime}$ is observed after taking action $u$ in state $s$ and obtaining reward $r$. However, IQL does not address the non-stationarity introduced due to the changing policies of the learning agents. Thus, unlike single-agent DQN, there is no guarantee of convergence even at the limit of infinite exploration.

\subsection{VDNs}
By contrast, Value decomposition networks (VDNs) \footnote{VDN code: \url{https://github.com/oxwhirl/pymarl}} \cite{sunehag2017valuedecomposition} seek to learn a joint action-value function $Q_{tot}(\boldsymbol{\tau}, \mathbf{u})$, where $\boldsymbol{\tau} \in \mathbf{T} \equiv \mathcal{T}^{n}$ is a joint action-
observation history and $\mathbf{u}$ is a joint action. It represents $Q_{tot}$ as the sum of individual value functions  $Q_{a}\left(\tau^{i}, u^{i} ; \theta^{i}\right)$:
$$
Q_{tot}(\boldsymbol{\tau}, \mathbf{u})=\sum_{i=1}^{n} Q_{i}\left(\tau^{i}, u^{i} ; \theta^{i}\right) .
$$

\subsection{Qatten}
\textbf{Qatten} \footnote{Qatten code: \url{https://github.com/simsimiSION/pymarl-algorithm-extension-via-starcraft}} \cite{yang2020qatten}, introduces an attention mechanism into the monotonic mixing network of QMIX:

\vspace{-0.5cm}

\begin{eqnarray}
Q_{tot} \approx c(s)+\sum_{h=1}^{H} w_{h} \sum_{i=1}^{N} \lambda_{i, h} Q^{i}
\end{eqnarray}

\vspace{-0.2cm}

\begin{eqnarray}
\lambda_{i, h} \propto \exp \left(e_{i}^{T} W_{k, h}^{T} W_{q, h} e_{s}\right)
\end{eqnarray}

where $w_{h}=\left|f^{N N}(s)\right|_{h}$, $W_{q, h}$ transforms $e_{s}$ into a global query, and $W_{k, h}$ transforms $e_{i}$ into an individual key. The $e_{s}$ and $e_{i}$ may be obtained by an embedding transformation layer for the true global state $s$ and the individual state $s_i$.

\subsection{QPLEX}
\textbf{QPLEX} \footnote{QPLEX code: \url{https://github.com/wjh720/QPLEX}} \cite{wang2020qplex} decomposes Q values into advantages and values based on Qatten, similar to Dueling-DQN \cite{wang2016dueling}:

\begin{equation}
\begin{aligned}
\text { (Joint Dueling) } Q_{tot}(\tau, u)=V_{tot}(\tau)+A_{tot}(\tau, u) \\
V_{tot}(\boldsymbol{\tau})=\max _{u^{\prime}} Q_{tot}\left(\boldsymbol{\tau}, \boldsymbol{u}^{\prime}\right)
\end{aligned}
\end{equation}

\begin{equation}
\begin{aligned}
\text { (Individual Dueling) } Q_{i}\left(\tau_{i}, u_{i}\right)=V_{i}\left(\tau_{i}\right)+A_{i}\left(\tau_{i}, u_{i}\right) \\
V_{i}\left(\tau_{i}\right)=\max _{u^{\prime}} Q_{i}\left(\tau_{i}, u_{i}^{\prime}\right)
\end{aligned}
\end{equation}

\begin{eqnarray}
\frac{\partial A_{tot}(s, \boldsymbol{u} ; \boldsymbol{\theta}, \phi)}{\partial A_{i}\left(\tau^{i}, u^{i}; \theta^{i}\right)} \geq 0, \quad \forall i \in \mathcal{N}
\label{eqn:mono2}
\end{eqnarray}

In other words, Eq. \ref{eqn:mono2} (advantage-based monotonicity) transfers the monotonicity constraint from Q values to advantage values. QPLEX thereby reduces limitations on the mixing network's expressiveness. 

\subsection{WQMIX}
\label{appendix:wqmix}
\textbf{WQMIX} \footnote{WQMIX code: \url{https://github.com/oxwhirl/wqmix}} \cite{rashid2020weighted}, just like Optimistically-Weighted QMIX (OW-QMIX), uses different weights for each sample to calculate the squared TD error of QMIX:

\begin{equation} 
\mathcal{L}(\theta) = \sum_{i=1}^{b} w(s, \mathbf{u})\left(Q_{tot}(\boldsymbol{\tau}, \mathbf{u}, s)-y_{i}\right)^{2}
\end{equation}

\begin{equation}
w(s, \mathbf{u})=\left\{\begin{array}{ll}
1 & Q_{tot}(\boldsymbol{\tau}, \mathbf{u}, s)<y_{i} \\
\alpha & \text { otherwise. }
\end{array}\right.
\end{equation}

Where $\alpha \in (0, 1]$ is a hyperparameter and $y_i$ is the true target Q value. WQMIX prefers those optimistic samples (true returns are larger than predicted), i.e., decreasing the weights of samples with non-optimistic returns. More critically, WQMIX uses an unconstrained true Q Network as a target network to guide the learning of QMIX. The authors prove that this approach can resolve the estimation errors of QMIX in the non-monotonic case.

\subsection{LICA} \label{lica}


\textbf{LICA} \footnote{LICA code: \url{https://github.com/mzho7212/LICA}} \cite{zhou2020learning} completely removes the monotonicity constraint through a policy mixing critic, as shown in Figure \ref{fig:lica}:

\begin{figure}[htbp]
  \centering
  \includegraphics[height=2.4cm]{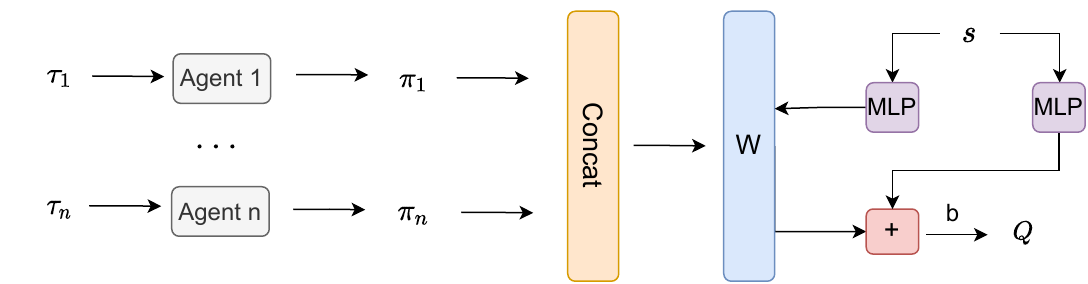}
  \caption{Architecture for LICA. LICA's mixing critic maps policy distribution to the Q value directly, in effect obviating the monotonicity constraint.}
  \label{fig:lica}
\end{figure}

LICA's mixing critic is trained using squared TD error. With a trained critic estimate, decentralized policy networks may then be optimized end-to-end simultaneously by maximizing $Q_{\theta_{c}}^{\pi}$ with the stochastic policies $\pi_{\theta_{i}}^{i}$ as inputs:

\begin{eqnarray}
\begin{aligned}
\max _{\theta} \mathbb{E}_{t, s_{t}, u_{t}^{1}, \ldots, \tau_{t}^{n}}[Q_{\theta_{c}}^{\pi}\left(s_{t}, \pi_{\theta_{1}}^{1}\left(\cdot \mid \tau_{t}^{1}\right), \ldots, \pi_{\theta_{n}}^{n}\left(\cdot \mid \tau_{t}^{n}\right)\right)
-\mathbb{E}_{i}\left[\mathcal{H}\left(\pi_{\theta_{i}}^{i}\left(\cdot \mid \tau_{t}^{i}\right)\right)\right]]
\end{aligned} 
\label{eqn:lica}
\end{eqnarray}

where the gradient of entropy item  $\mathbb{E}_{i}\left[\mathcal{H}\left(\pi_{\theta_{i}}^{i}\left(\cdot \mid z_{t}^{i}\right)\right)\right]]$ is normalized by taking the quotient of its own modulus length: Adaptive Entropy (Adapt Ent). Adaptive Entropy automatically adjusts the coefficient of entropy loss in different scenarios.

\subsection{VMIX} 
\label{appendix:vmix}
\input{fig/vmix1}
VMIX \footnote{VMIX code: \url{https://github.com/hahayonghuming/VDACs}}  \cite{su2020value} combines the Advantage Actor-Critic (A2C) \cite{stooke2018accelerated} with QMIX to extend the monotonicity constraint to value networks (not Q value network), as shown in Eq. \ref{eqn:vmix} and Figure \ref{fig:vmix_net}. We verified that the monotonicity constraint also has a positive effect on the value network based on VMIX (Figure \ref{fig:vmix}).

\begin{equation}
\begin{aligned}
&V_{tot}(s; \boldsymbol{\theta}, \phi)
\\= &g_{\phi}\left(s, V^{1}\left(\tau^{1} ; \theta^{1}\right), \ldots, V^{N}\left(\tau^{N};  \theta^{N}\right)\right)
\end{aligned}
\end{equation}

\begin{equation} \label{eqn:vmix}
\frac{\partial V_{tot}}{\partial V^{i}} \geq 0, \quad \forall i \in\{1, \ldots, N\}
\end{equation}

where $\phi$ is the parameter of value mixing network, and $\theta_i$ is the parameter of agent network. With the centralized value function $ V_{tot}$, the policy networks can be trained by policy gradient (Eq. \ref{p_gradient}),

\begin{equation}
\hat{g}_{i}=\left.\frac{1}{\left|\mathcal{D}\right|} \sum_{\tau \in \mathcal{D}} \sum_{t=0}^{T} \nabla_{\theta} \log \pi_{\theta^i}\left(u_{t}^i \mid \tau_{t}^i\right)\right|_{\theta^{i}} \hat{A}_{t}
\label{p_gradient}
\end{equation}

where $\hat{A}_{t} = r + V_{tot}(s') - V_{tot}(s)$ is the advantage function, and $\mathcal{D}$ denotes replay buffer.

\subsection{Relationship between Previous Works}
\label{appendix:Relationships}
VDNs requires a linear decomposition of Q values, so it has the strongest monotonicity constraint. Since the weights calculated by softmax (attention mechanism) are greater than or equal to zero, the constraint strengths of Qatten and QMIX are approximately equal. QPLEX just shifts the constraint to advantage values without removing it. WQMIX relaxes the monotonicity constraint even further by a true Q value network and theoretical guarantees. LICA completely removes the monotonicity constraint by new network architecture. We rank the strength of the monotonicity constraints on these MARL algorithms: 
\begin{equation}
     \text{VDNs} >  \text{QMIX} \approx  \text{Qatten} >  \text{QPLEX} >  \text{WQMIX} >  \text{LICA}
\end{equation}

%% file: fig/vmix1.tex
\begin{figure}[h]
  \centering
  \includegraphics[height=2.4cm]{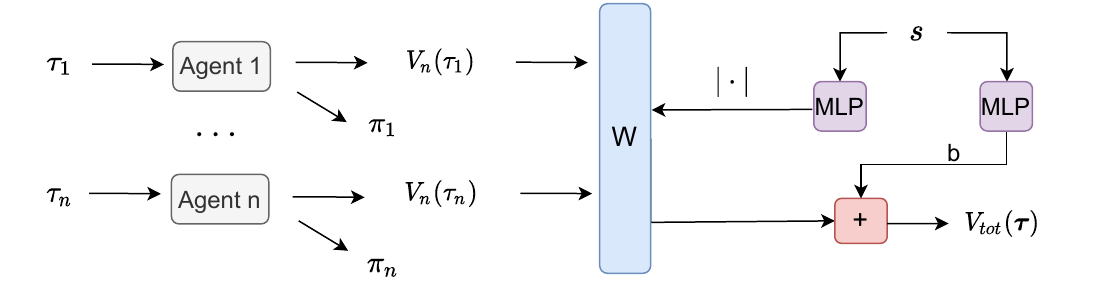}
    \caption{Architecture for VMIX: $|\cdot| $ denotes \textbf{absolute value operation}, decomposing $V_{tot}$ into $V_i$.}
  \label{fig:vmix_net}
\end{figure}